\documentclass{article}

\usepackage[numbers]{natbib} 

\usepackage[preprint]{neurips_2026}

\usepackage{amsmath}
\usepackage{amssymb}
\usepackage{mathtools}
\usepackage{amsthm}

\newtheorem{proposition}{Proposition}
\newtheorem*{proposition*}{Proposition}

\newtheorem{theorem}{Theorem}
\newtheorem*{theorem*}{Theorem}  

\theoremstyle{plain}

\newtheorem{corollary}[theorem]{Corollary}
\theoremstyle{definition}
\newtheorem{definition}[theorem]{Definition}
\newtheorem{assumption}[theorem]{Assumption}
\theoremstyle{remark}
\newtheorem{remark}[theorem]{Remark}

\usepackage[utf8]{inputenc} 
\usepackage[T1]{fontenc}    
\usepackage{hyperref}       
\usepackage{url}            
\usepackage{booktabs}       
\usepackage{amsfonts}       
\usepackage{nicefrac}       
\usepackage{microtype}      
\usepackage{xcolor}         
\usepackage{algorithm}
\usepackage{algorithmic}
\usepackage{wrapfig}        

\usepackage[capitalize,noabbrev]{cleveref}

\title{Safe Evolution with Circuit Anchors}

\author{
  Yan Liu \\
  Chinese University of Hong Kong\\
  \texttt{runningmelles@gmail.com} \\
  \And
  Jie Fu \\
  IQuest Research \\
  \texttt{} \\
  \\
  \And
  Tsung-Yi Ho \\
  Chinese University of Hong Kong \\
  \texttt{tyho@cse.cuhk.edu.hk} \\
}

\begin{document}

\maketitle

\begin{abstract}
In biological evolution, unconstrained mutation can lead to 
catastrophic outcomes: organisms may evolve enhanced capabilities 
while losing essential functions for survival. Nature's solution 
is \textit{developmental constraints}, where core regulatory genes 
remain anchored while peripheral genes adapt freely. We observe 
that current self-evolution algorithms for large language models 
lack analogous constraints. They optimize purely for capability, 
implicitly assuming safety will be preserved. Our experiments 
reveal this assumption to be dangerously wrong: models can 
\textit{misevolve} into powerful yet dangerous entities.
Inspired by how Hox genes anchor body structure across $500$ million 
years of evolution, we propose \textbf{Circuit-Anchored Evolution (CAE)}. 
Using mechanistic interpretability, we identify a tiny 
\textit{safety circuit}, comprising less than $2$\% of model features, 
that causally mediates safety behaviors. We anchor this circuit 
during evolution, constraining it within a small displacement bound 
while allowing the remaining features to evolve freely. This 
mirrors the biological principle of \textit{evolvability with 
constraint}: preserving what is essential while adapting what is 
peripheral.
Experiments across $3$ model families and two evolution algorithms 
demonstrate that CAE achieves superior safety 
preservation with minimal capability loss, substantially outperforming 
explicit reward-based constraints in both effectiveness and efficiency. Just as developmental constraints 
prevent biological evolution from producing nonviable organisms, 
circuit anchoring prevents model evolution from producing capable 
but dangerous systems.
\end{abstract}

\section{Introduction}
Evolution is a double-edged sword \cite{miconi2008evolution,hanley2011double,perc2015double}. In biology, it has produced remarkable adaptations, from the eagle's eye to the human brain \cite{emery2005evolution,williams2023eagle}. 
But unconstrained evolution can also produce monsters: organisms with enhanced capabilities but fatal deficiencies \cite{alberch1989logic}. A mutation that improves muscle strength means nothing if it simultaneously disrupts heart development.
Nature's solution for safe evolution is elegant: anchor critical genes during evolution.
A small set of \textit{master regulatory genes}, such as the Hox genes controlling body plan \cite{hughes2002hox,mallo2010hox}, remains fiercely conserved across hundreds of millions of years \cite{pearson2005modulating}. These genetic anchors ensure that evolution explores new capabilities 
without compromising fundamental viability. Mutations in Hox genes are almost always lethal \cite{goodman2001human}, eliminated by purifying selection \cite{brunet2021role} before they can propagate. This principle of \textit{evolvability with constraint} \cite{sharov2014evolutionary} is what allows species to adapt without self-destructing.

Artificial intelligence is now undergoing its own evolutionary process \cite{gao2025survey,tao2024surveyselfevolutionlargelanguage}.
The pursuit of artificial general intelligence (AGI) has increasingly embraced self-evolutionary paradigms \cite{dr0,self-evol1}.
Recent advances demonstrate that such approaches can yield substantial gains in reasoning \cite{feng2025group,novikov2025alphaevolve}, surpassing models trained on static human-curated datasets. 
However, this progress introduces a critical yet overlooked risk. Current self-evolution algorithms operate in an entirely unconstrained manner, optimizing solely for task performance without any mechanism to preserve alignment with human values. 
\begin{wrapfigure}{r}{0.48\textwidth}
  \vspace{-18pt}
  \begin{center}
    \includegraphics[width=\linewidth]{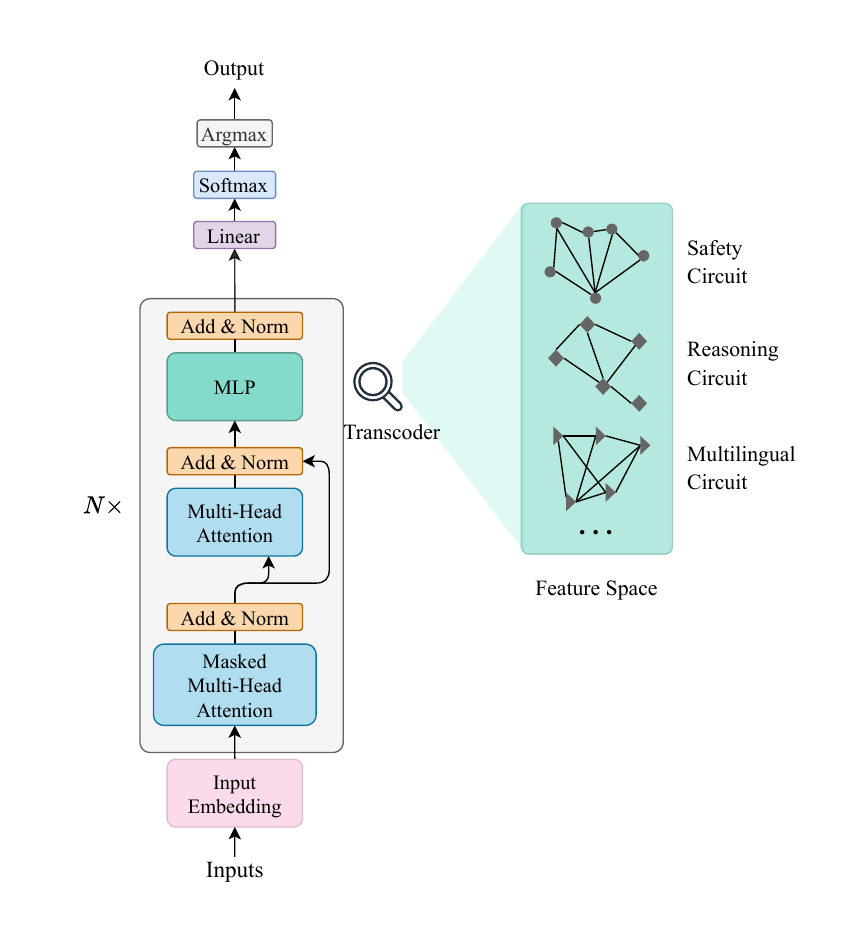}
  \end{center}
  \vspace{-10pt}
  \caption{Overview of our approach. The Transcoder decomposes MLP activations into interpretable features, revealing functionally distinct circuits. CAE preserves the Safety Circuit via targeted constraints while allowing other circuits to evolve freely.}
  \label{fig:main}
  \vspace{-10pt}
\end{wrapfigure}
This creates significant potential for \textit{misevolution}, where models evolve in unintended directions that produce undesirable or even harmful behaviors~\citep{misevolve}. 
As models grow increasingly powerful through self-evolution, this unconstrained evolution poses a serious threat. Without proper safeguards, we risk creating systems that are highly capable and fundamentally dangerous.

Recent advances in mechanistic interpretability reveal that model behaviors are localized in specific circuits \cite{tigges2024llmcircuitanalysesconsistent}. The biological analogy motivates our approach: \textbf{Circuit-Anchored Evolution (CAE)}. 
As shown in Figure \ref{fig:main}, using transcoder-based analysis \cite{transcoder}, we identify a tiny \textit{safety circuit}: a sparse set of features comprising less than 2\% of the total, yet causally responsible for safe behaviors.
Like Hox genes, this safety circuit is tiny but critical; disrupting it leads to catastrophic results.
Rather than treating safety as an optimization objective that competes with capability, we \textit{anchor} the safety circuit in evolution by constraining the KL divergence.
This constraint acts as an artificial purifying selection, preventing ``mutations'' that would disrupt the safety 
circuit. The remaining features, like peripheral genes, are free to evolve without constraint. 
Our circuit anchoring represents an \textit{implicit} constraint 
operating on internal structure, in contrast to \textit{explicit} 
constraints that supervise external behavior via reward models. 
To understand the trade-offs, we systematically compare both paradigms.
Through extensive experiments on different models and evolution algorithms, we find that implicit anchoring substantially outperforms 
explicit supervision. Circuit anchoring achieves higher safety 
preservation with minimal capability loss and significantly lower 
computational overhead. 
Just as developmental constraints enable biological species to evolve new adaptations without losing viability, circuit anchoring enables language models to evolve new capabilities without 
losing safety.

\section{Preliminaries}
\label{preliminaries}

\subsection{Notation}

Let $\mathcal{M}_\theta: \mathcal{X} \to \mathcal{Y}$ denote a large 
language model parameterized by $\theta \in \Theta$, where $\mathcal{X}$ 
is the space of input sequences and $\mathcal{Y}$ is the space of 
output sequences. For an input $x \in \mathcal{X}$, we denote the 
hidden state at layer $l \in [L] := \{1, \ldots, L\}$ as 
$h_l(x; \theta) \in \mathbb{R}^d$. We use $\pi_\theta(\cdot | x)$ 
to denote the output distribution over next tokens. Throughout, 
we use $\|\cdot\|$ to denote the $\ell_2$ norm and $D_{KL}(\cdot \| \cdot)$ 
to denote the Kullback-Leibler divergence.

\subsection{Transcoders and Feature Decomposition}

We adopt the transcoder framework introduced by \citet{transcoder}, 
which provides an interpretable decomposition of MLP activations into 
sparse feature vectors.

\begin{definition}[Transcoder] 
\label{def:transcoder} 
A transcoder $\mathcal{T}_l = (E_l, D_l)$ for layer $l$ consists of an encoder $E_l: \mathbb{R}^d \to \mathbb{R}^K$ and a decoder $D_l: \mathbb{R}^K \to \mathbb{R}^d$, trained to satisfy $h_l(x; \theta) \approx D_l(E_l(h_l(x; \theta))) + \epsilon_l(x)$, where $\epsilon_l(x)$ is a residual error term. The encoder produces a sparse feature activation vector $f_l(x; \theta) := E_l(h_l(x; \theta)) \in \mathbb{R}^K$, where $K \gg d$ and $\|f_l(x; \theta)\|_0 \ll K$. 
\end{definition}

\begin{definition}[Cross-Layer Feature Vector] 
\label{def:cross_layer} 
For a model with $L$ layers and transcoders $\{\mathcal{T}_l\}_{l=1}^L$, we define the concatenated feature vector as $f(x; \theta) := \bigoplus_{l=1}^{L} f_l(x; \theta) \in \mathbb{R}^{LK}$, where $\oplus$ denotes concatenation. We index individual features as $f^{(l,k)}(x; \theta)$ for layer $l$ and feature index $k$. 
\end{definition}

\subsection{Self-Evolutionary Training}

We consider the general framework of self-evolutionary training, 
where a model iteratively improves by learning from self-generated 
data. We first formalize two representative algorithms, then show 
they share a common mathematical structure.

\subsubsection{GRPO-Based Evolution (EVOL-RL)}

Following \cite{shao2024deepseekmath}, Group Relative Policy 
Optimization (GRPO) generates $G$ responses $\{o_1, \ldots, o_G\}$ 
for each prompt $q$ and computes normalized advantages:
\begin{equation}
    \hat{A}_i = \frac{r_i - \mathrm{mean}(r_1, \ldots, r_G)}{\mathrm{std}(r_1, \ldots, r_G)}
\end{equation}
The policy is updated via the clipped surrogate objective:
\begin{equation}
\label{eq:grpo}
\resizebox{0.7\linewidth}{!}{$\displaystyle
    \mathcal{L}_{\text{GRPO}}(\theta) = \frac{1}{G} \sum_{i=1}^{G} 
    \frac{1}{|o_i|} \sum_{t=1}^{|o_i|} \min\left\{ \rho_{i,t} \hat{A}_{i,t}, 
    \mathrm{clip}(\rho_{i,t}, 1-\epsilon, 1+\epsilon) \hat{A}_{i,t} \right\}
$}
\end{equation}
where $\rho_{i,t} = \frac{\pi_\theta(o_{i,t} | q, o_{i,<t})}{\pi_{\theta_{\text{old}}}(o_{i,t} | q, o_{i,<t})}$ 
is the importance ratio. EVOL-RL applies this objective iteratively 
on self-generated reasoning tasks.

\subsubsection{Zero Data Evolution (Abs-ZERO)}

The Abs-ZERO framework \citep{abs0} jointly trains a proposer 
policy $\pi_\theta^{\text{propose}}$ and a solver policy 
$\pi_\theta^{\text{solve}}$ with objective:
\begin{equation}
\label{eq:abszero}
\resizebox{0.7\linewidth}{!}{$\displaystyle
    \mathcal{J}_{\text{AZ}}(\theta) = \mathbb{E}_{z \sim p(z)} \left[
    \mathbb{E}_{\tau, (x, y^*)} \left[ r_e^{\text{propose}}(\tau, \pi_\theta)
    + \lambda \mathbb{E}_{y \sim \pi_\theta^{\text{solve}}} \left[
    r_e^{\text{solve}}(y, y^*) \right] \right] \right]
$}
\end{equation}
where $\tau$ is a proposed task, $(x, y^*)$ is the task-answer pair, 
and $r_e^{\text{propose}}$, $r_e^{\text{solve}}$ are rewards for 
task proposal and solving respectively.

\subsubsection{Unified Formulation}
\label{sec:unified_formulation}

Despite their different designs, both algorithms share a common structure: iteratively optimizing model parameters to maximize a task-specific reward signal. We abstract this as $\theta^{(t+1)} = \theta^{(t)} + \eta \nabla_\theta \mathcal{L}_{\text{evol}}(\theta^{(t)})$, where $\mathcal{L}_{\text{evol}}$ represents the evolution objective, which can be instantiated as $\mathcal{L}_{\text{GRPO}}$ or $\mathcal{J}_{\text{AZ}}$ depending on the algorithm.
More generally, our framework applies to any self-evolution algorithm that satisfies three conditions: (1) it optimizes model parameters $\theta$ via gradient-based updates, (2) the optimization objective $\mathcal{L}_{\text{evol}}$ is differentiable with respect to $\theta$, and (3) it does not explicitly constrain safety-related behaviors.

\subsection{Problem Formulation: Safety Degradation}

We now formalize the phenomenon of safety degradation during evolution.

\begin{definition}[Safety Behavior] 
\label{def:safety} 
Let $\mathcal{D}_{\text{harm}} = \{x_i^{\text{harm}}\}_{i=1}^n$ be a set of harmful prompts. A model $\mathcal{M}_\theta$ exhibits safe behavior if for all $x \in \mathcal{D}_{\text{harm}}$, we have $\pi_\theta(y_{\text{refuse}} | x) \geq \tau$, where $y_{\text{refuse}}$ denotes refusal responses and $\tau \in (0, 1)$ is a safety threshold. 
\end{definition}

\begin{definition}[Safety Degradation]
\label{def:safety_degradation}
Let $\theta_0$ be a safely-aligned model satisfying 
Definition~\ref{def:safety}. We say evolutionary training induces 
\textit{safety degradation} if there exists $T > 0$ such that:
\begin{equation}
\resizebox{0.6\linewidth}{!}{$\displaystyle
    \mathbb{E}_{x \sim \mathcal{D}_{\text{harm}}} \left[
    \pi_{\theta^{(T)}}(y_{\text{refuse}} | x) \right] <
    \mathbb{E}_{x \sim \mathcal{D}_{\text{harm}}} \left[
    \pi_{\theta_0}(y_{\text{refuse}} | x) \right] - \delta
$}
\end{equation}
for some $\delta > 0$.
\end{definition}
Our goal is to design an evolutionary framework that maximizes task 
performance while provably bounding safety degradation.

\begin{algorithm}[t]
\caption{Circuit-Anchored Evolution (CAE)}
\label{alg:cae}
\begin{algorithmic}[1]
\REQUIRE Safely-aligned model $\theta_0$, transcoders $\{\mathcal{T}_l\}_{l=1}^L$, 
         safety circuit $\mathcal{S}$, constraint weight $\lambda$, 
         learning rate $\eta$, evolution objective $\mathcal{L}_{\text{evol}}$
\ENSURE Evolved model $\theta^{(T)}$ with preserved safety
\STATE Initialize $\theta \leftarrow \theta_0$
\STATE Cache reference activations: $\{f_{\mathcal{S}}(x; \theta_0)\}_{x \in \mathcal{D}_{\text{ref}}}$
\FOR{$t = 1$ to $T$}
    \STATE \textcolor{gray}{// Standard evolution step}
    \STATE Compute evolution gradient: $g_{\text{evol}} \leftarrow 
           \nabla_\theta \mathcal{L}_{\text{evol}}(\theta)$
    \STATE \textcolor{gray}{// Circuit anchoring step}
    \STATE Sample $\{x_m\}_{m=1}^M$ from reference distribution
    \FOR{each $x_m$}
        \STATE Compute current activations via frozen transcoders: 
               $f_{\mathcal{S}}(x_m; \theta)$
        \STATE Compute KL: $\ell_m \leftarrow \sum_{(l,k) \in \mathcal{S}} 
               D_{KL}(f^{(l,k)}(x_m; \theta_0) \| f^{(l,k)}(x_m; \theta))$
    \ENDFOR
    \STATE Compute anchor gradient: $g_{\text{anchor}} \leftarrow 
           \frac{1}{M} \sum_m \nabla_\theta \ell_m$
    \STATE \textcolor{gray}{// Combined update}
    \STATE $\theta \leftarrow \theta + \eta \left( g_{\text{evol}} 
           - \lambda \cdot g_{\text{anchor}} \right)$
\ENDFOR 

\RETURN $\theta^{(T)}$
\end{algorithmic}
\end{algorithm}

\section{Circuit-Anchored Evolution}
We summarize the complete CAE procedure in Algorithm~\ref{alg:cae}.

\subsection{Safety Circuit Identification}

We leverage the circuit tracing methodology \citep{hanna2025circuit} 
to identify features causally responsible for safety.

\begin{definition}[Attribution Score]
\label{def:attribution}
For a feature $(l, k)$ and target logit $y$, the direct attribution 
score is defined as:
\begin{equation}
    \alpha^{(l,k)}_y(x; \theta) := \frac{\partial \log \pi_\theta(y | x)}{\partial f^{(l,k)}(x; \theta)} \cdot f^{(l,k)}(x; \theta)
\end{equation}
This measures the causal contribution of feature $(l, k)$ to the 
probability of output $y$.
\end{definition}

\begin{definition}[Safety Circuit]
\label{def:safety_circuit}
Given a safely-aligned model $\theta_0$, a harmful prompt set 
$\mathcal{D}_{\text{harm}}$, and a threshold $\gamma > 0$, the safety circuit is defined as:
\begin{equation}
    \mathcal{S} := \left\{ (l, k) : \mathbb{E}_{x \sim \mathcal{D}_{\text{harm}}} 
    \left[ \alpha^{(l,k)}_{y_{\text{refuse}}}(x; \theta_0) \right] \geq \gamma \right\}
\end{equation}
We denote $f_{\mathcal{S}}(x; \theta) := \{f^{(l,k)}(x; \theta)\}_{(l,k) \in \mathcal{S}}$ 
as the safety feature vector.
\end{definition}

\begin{remark}
Safety circuit $\mathcal{S}$ is computed once from the aligned model $\theta_0$ and remains fixed in evolution. This is justified by the finding that circuit structure is largely preserved across 
fine-tuning~\cite{tigges2024llmcircuitanalysesconsistent}.
\end{remark}

\subsection{Feature-Space Anchoring}

We now introduce our core contribution: a constraint that preserves safety circuit activations during evolution.

\begin{definition}[Circuit Activation Distribution]
\label{def:activation_dist}
For a given input distribution $p(x)$ and safety circuit $\mathcal{S}$, we define the circuit activation distribution as:
\begin{equation}
    P_\theta^{\mathcal{S}}(f_{\mathcal{S}}) := \mathbb{E}_{x \sim p(x)} 
    \left[ \mathbf{1}[f_{\mathcal{S}}(x; \theta) = f_{\mathcal{S}}] \right]
\end{equation}
In practice, we treat $f_{\mathcal{S}}(x; \theta)$ as samples from 
this distribution.
\end{definition}

\begin{definition}[Circuit-Anchored Objective]
\label{def:cao}
The Circuit-Anchored Evolution (CAE) objective augments the base 
evolutionary loss with a circuit-level KL constraint:
\begin{equation}
\label{eq:cao}
    \mathcal{L}_{\text{CAE}}(\theta) := \mathcal{L}_{\text{evol}}(\theta) 
    - \lambda \cdot \mathcal{L}_{\text{anchor}}(\theta)
\end{equation}
where
\begin{equation}
\label{eq:anchor_loss}
\resizebox{0.65\linewidth}{!}{$\displaystyle
\begin{aligned}
    \mathcal{L}_{\text{anchor}}(\theta) &:= D_{KL}\left( P_{\theta_0}^{\mathcal{S}}
    \| P_{\theta}^{\mathcal{S}} \right) \\
    &= \mathbb{E}_{x \sim p(x)} \left[
    \sum_{(l,k) \in \mathcal{S}} D_{KL}\left( f^{(l,k)}(x; \theta_0) \|
    f^{(l,k)}(x; \theta) \right) \right]
\end{aligned}
$}
\end{equation}
and $\lambda > 0$ is a hyperparameter controlling the constraint strength.
\end{definition}

\subsection{Instantiation for Specific Algorithms}

The CAE framework is algorithm-agnostic and can be instantiated for 
any self-evolution algorithm satisfying the conditions in 
Section~\ref{sec:unified_formulation}.

\paragraph{CAE for GRPO-Based Evolution.}
For EVOL-RL, we augment the GRPO objective (Eq.~\ref{eq:grpo}):
\begin{equation}
    \mathcal{L}_{\text{CAE-GRPO}}(\theta) = \mathcal{L}_{\text{GRPO}}(\theta) 
    - \lambda \cdot \mathcal{L}_{\text{anchor}}(\theta)
\end{equation}

\paragraph{CAE for Abs-ZERO.}
For Abs-ZERO, we augment the joint objective (Eq.~\ref{eq:abszero}):
\begin{equation}
    \mathcal{J}_{\text{CAE-AZ}}(\theta) = \mathcal{J}_{\text{AZ}}(\theta) 
    - \lambda \cdot \mathcal{L}_{\text{anchor}}(\theta)
\end{equation}
The anchor constraint is applied to both proposer and solver policies, 
as they share underlying parameters.

\paragraph{General Applicability.}
More broadly, CAE applies to any evolution algorithm of the form 
$\theta^{(t+1)} = \theta^{(t)} + \eta \nabla_\theta \mathcal{L}_{\text{evol}}(\theta^{(t)})$ 
by simply adding the anchor term:
\begin{equation}
    \theta^{(t+1)} = \theta^{(t)} + \eta \nabla_\theta \left( 
    \mathcal{L}_{\text{evol}}(\theta^{(t)}) - \lambda \cdot 
    \mathcal{L}_{\text{anchor}}(\theta^{(t)}) \right)
\end{equation}
This includes Self-Play Fine-Tuning, Self-Rewarding LMs, iterative DPO, and future gradient-based evolution methods.

\subsection{Theoretical Analysis}

We now establish key theoretical properties of the CAE objective.

\begin{assumption} 
\label{ass:regularity} 
We assume the following regularity conditions: \textbf{(1) Bounded activations:} There exists $B > 0$ such that $\|f(x; \theta)\|_\infty \leq B$ for all $x, \theta$. \textbf{(2) Lipschitz continuity:} The feature map $\theta \mapsto f_{\mathcal{S}}(x; \theta)$ is $L_f$-Lipschitz for all $x$. \textbf{(3) Almost-everywhere differentiability:} The encoder $E_l$ is differentiable almost everywhere with bounded Jacobian $\|J_{E_l}\| \leq M$. 
\end{assumption}

\begin{remark}[Justification of Assumption~\ref{ass:regularity}]
\label{rem:relu_justification}
Condition 3 is naturally satisfied by standard transcoder architectures employing ReLU activations. ReLU is differentiable everywhere except at exactly zero, which occurs with probability zero under continuous input distributions. Furthermore, its Jacobian is a diagonal matrix with entries in $\{0, 1\}$, strictly satisfying the bounded Jacobian condition with $M=1$. Analyzing such networks with gradients defined almost everywhere is standard in deep learning theory.
\end{remark}

\begin{proposition}[Gradient of Anchor Loss] 
\label{prop:gradient} 
Under Assumption~\ref{ass:regularity} (Condition 3), the gradient of the anchor 
loss admits the form:
\begin{equation}
\resizebox{0.65\linewidth}{!}{$\displaystyle
    \nabla_\theta \mathcal{L}_{\text{anchor}}(\theta) = \mathbb{E}_{x \sim p(x)}
    \left[ \sum_{(l,k) \in \mathcal{S}} \left( 1 + \log \frac{f^{(l,k)}(x; \theta)}{f^{(l,k)}(x; \theta_0)} \right)
    \nabla_\theta f^{(l,k)}(x; \theta) \right]
$}
\end{equation}
where the gradient flows through the frozen transcoder encoder.
\end{proposition}

\begin{proof}
By the chain rule and the definition of KL divergence for continuous 
distributions. The transcoder weights are frozen, so gradients pass 
through $E_l$ to $h_l$ and subsequently to $\theta$. Full derivation 
in Appendix~\ref{app:gradient_proof}.
\end{proof}

\subsection{Safety Preservation Guarantee}

We now prove that CAE provides a formal guarantee of safety preservation.

\begin{theorem}[Conditional Safety Bound] 
\label{thm:safety_bound} 
Let $\theta_0$ be a safely-aligned model and $\theta^{(T)}$ be the 
model after $T$ steps of CAE training with constraint weight $\lambda$. 
Under Assumption~\ref{ass:regularity} (Conditions 1 and 2), if the anchor loss is bounded 
as $\mathcal{L}_{\text{anchor}}(\theta^{(T)}) \leq \epsilon$, then: 
\begin{equation} 
    \left| \mathbb{E}_{x \sim \mathcal{D}_{\text{harm}}} \left[ 
    \pi_{\theta^{(T)}}(y_{\text{refuse}} | x) - \pi_{\theta_0}(y_{\text{refuse}} | x) 
    \right] \right| \leq C \sqrt{\epsilon} 
\end{equation} 
where $C$ is a constant depending on $B$, $L_f$, and the circuit 
size $|\mathcal{S}|$. 
\end{theorem} 

\begin{remark}[Separation of Theory and Empirical Faithfulness]
\label{rem:theory_vs_empirical}
We emphasize that Theorem~\ref{thm:safety_bound} is a \textit{conditional} statement: it bounds the change in refusal probability \textit{given} a bounded anchor loss. Crucially, the proof relies only on Pinsker's inequality and the Lipschitz property (Condition 2), completely relaxing the need for encoder differentiability (Condition 3). Furthermore, the assumption that the identified circuit genuinely mediates safety (i.e., circuit faithfulness) is not a theoretical premise, but an empirical property that we rigorously verify through causal intervention experiments (see Figure~\ref{fig:circuit_validity}). Full proof is deferred to Appendix~\ref{app:safety_proof}.
\end{remark}

\begin{corollary}
\label{cor:no_degradation}
If $\lambda$ is chosen such that $\mathcal{L}_{\text{anchor}}(\theta^{(t)}) 
\leq \epsilon$ for all $t \leq T$, then CAE prevents safety degradation 
(Definition~\ref{def:safety_degradation}) with $\delta = C\sqrt{\epsilon}$.
\end{corollary}

\subsection{Comparison with Explicit Safety Constraints}

We contrast our implicit circuit-based constraint with explicit 
behavioral constraints via reward models.

\begin{definition}[Reward-Based Safety Constraint]
\label{def:reward_constraint}
Given a safety reward model $R_{\text{safe}}: \mathcal{X} \times 
\mathcal{Y} \to \mathbb{R}$, the reward-based evolution objective is:
\begin{equation}
    \mathcal{L}_{\text{RM}}(\theta) = \mathcal{L}_{\text{evol}}(\theta) 
    + \lambda_{\text{RM}} \mathbb{E}_{x, y \sim \pi_\theta} \left[ 
    R_{\text{safe}}(x, y) \right]
\end{equation}
\end{definition}

\begin{proposition}[Comparison of Constraint Paradigms] 
\label{prop:comparison} 
Let $d_{\text{out}}$ be the vocabulary size. Comparing circuit-based and reward-based constraints reveals three key differences: \textbf{(1) Dimensionality:} The reward constraint operates on $O(d_{\text{out}})$ output dimensions per token, while the circuit constraint operates on $O(|\mathcal{S}|)$ feature dimensions, where $|\mathcal{S}| \ll d_{\text{out}}$. \textbf{(2) Specificity:} The reward constraint penalizes all outputs deemed unsafe regardless of internal mechanism, whereas the circuit constraint specifically preserves the features causally responsible for safe behavior. \textbf{(3) Computational cost:} The reward constraint requires an additional forward pass through $R_{\text{safe}}$ for each generated sequence, while the circuit constraint requires only activation extraction through frozen transcoders, with a minimal cost of $O(|\mathcal{S}| \cdot d)$.
\end{proposition}

\section{Experiments}

\subsection{Settings}
\begin{wraptable}{r}{0.48\textwidth} 
  \vspace{-20pt} 
  \centering
  \caption{Strict separation of datasets for circuit tracing and causal verification.} 
  \label{tab:datasets} 
  \resizebox{\linewidth}{!}{
  \begin{tabular}{llc} 
    \toprule
    \textbf{Phase} & \textbf{Dataset} & \textbf{Num.} \\ 
    \midrule
    Circuit Tracing & Do-Not-Answer \cite{Do-Not-Answer} & 939 \\ 
    Circuit Tracing & DAN \cite{DAN} & 390 \\ 
    Circuit Tracing & Malicious Instruct \cite{MaliciousInstruct} & 100 \\ 
    \midrule
    Causal Verification & AdvBench \cite{zou2023universal} & 520 \\ 
    \bottomrule
  \end{tabular} 
  }
  \vspace{-10pt} 
\end{wraptable}
We conduct experiments on three instruction-tuned model families with available pre-trained transcoders: 
Qwen3-4B-Instruct, Gemma-2-2B-IT, and Llama-3.2-1B-Instruct. 
We choose instruction-tuned models as our starting point because they represent the realistic deployment setting where capability enhancement is most valuable, and they possess identifiable 
safety circuits that can be anchored during evolution.
We evaluate on two representative self-evolution algorithms: EVOL-RL~\citep{evol-rl}, a GRPO-based algorithm that improves mathematical reasoning through self-generated problems, and Abs-ZERO~\citep{abs0}, a joint proposer-solver framework achieving self-improvement with zero external data. For both algorithms, we use the original training and evaluation configurations and datasets as described in their respective papers, evolving each model for 500 steps with batch size 32.
We report the average accuracy across benchmarks as the overall capability score.
\begin{figure*}[ht]
  \vspace{-20pt}
  \begin{center}
    \centerline{\includegraphics[width=1\textwidth]{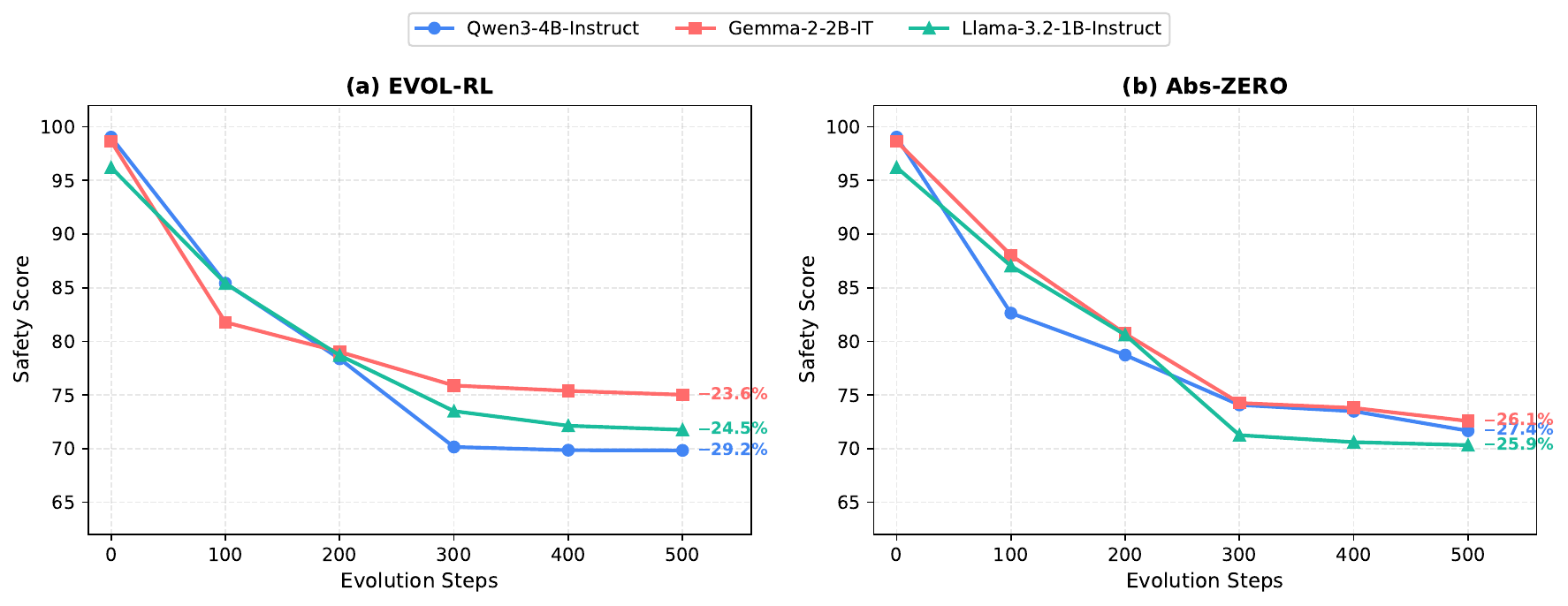}}
    \vspace{-10pt} 
    \caption{
      Safety degradation during unconstrained self-evolution. 
We track the refusal rate across $500$ evolution steps for three 
instruction-tuned models. All models exhibit severe safety 
collapse. This consistent pattern across model families demonstrates that safety 
degradation is an inherent risk of unconstrained evolution, 
not an artifact of specific architectures.
    }
    \label{fig:safety_decay}
  \end{center}
  \vspace{-20pt} 
\end{figure*}
For safety circuit location and evaluation, we adopt the framework from \cite{li2024safety}. To ensure the identified circuits capture universal safety mechanisms rather than dataset-specific artifacts, we employ a \textbf{multi-dataset intersection} pipeline with strict separation between tracing and testing phases (Table~\ref{tab:datasets}). For circuit tracing, we use a diverse corpus of 1,429 prompts spanning Do-Not-Answer \cite{Do-Not-Answer}, DAN \cite{DAN}, and Malicious Instruct \cite{MaliciousInstruct}. By tracing circuits across these diverse formats and taking their intersection, we effectively decouple the circuit from domain-specific vocabulary or syntactic patterns. Crucially, for causal verification and safety evaluation during evolution, we use an entirely held-out dataset, AdvBench \cite{zou2023universal} (520 queries). The fact that our identified circuit maintains strong causal effects on this unseen dataset strongly validates our intersection methodology. We also measure over-refusal rate using 1,000 benign queries from the Alpaca dataset, where a lower rate indicates better usability.

\subsection{Safety Degradation in Unconstrained Evolution}

Before introducing our method, we first quantify the severity of 
safety degradation in current self-evolution algorithms.
Figure~\ref{fig:safety_decay} tracks the refusal rate of three 
instruction-tuned models during $500$ steps of unconstrained 
evolution under two algorithms: EVOL-RL and 
Abs-ZERO. The pattern is striking and consistent across all 
settings. Models begin with near-perfect refusal rates ($96$--$99\%$) 
but experience rapid degradation within the first $200$ steps, 
losing approximately $15$--$20$ percentage points. By the end of 
evolution, refusal rates stabilize around $70$--$75\%$, representing 
a degradation of roughly $25$ percentage points.
While the final refusal rates remain above the $50\%$ threshold, 
this degradation is still concerning for two reasons. First, a 
model that refuses only $70\%$ of harmful requests will comply 
with nearly one-third of them, substantially increasing risk in 
deployment. Second, and more importantly, this degradation occurs 
after only $500$ evolution steps. Extended evolution, which is 
common in practice for maximizing capability gains, would likely 
lead to further safety collapse.
The consistency of this pattern across three model families (Qwen, Gemma, Llama) and two evolution algorithms indicates that safety degradation is an inherent risk of unconstrained 
self-evolution, not an artifact of specific architectures or 
training procedures. Notably, the degradation curves show similar 
trajectories regardless of the starting refusal rate or model 
size, suggesting a fundamental tension between capability 
optimization and safety preservation.
These findings motivate our central question: \textit{Can we 
enable capability evolution while preventing safety degradation?}

\subsection{Causal Verification of Safety Circuit}

To verify that the identified circuit is indeed causally responsible for safety behavior, we perform activation intervention experiments on the held-out AdvBench dataset. Specifically, we scale the activations of safety circuit features by factors ranging from 0 (complete suppression) to 10 (strong amplification) and measure the resulting refusal rate on these unseen harmful prompts. Figure~\ref{fig:circuit_validity} shows the results across three models. For Llama-3.2-2B-Instruct,  intervening on the safety circuit produces clear causal effects: suppressing 
activations (scale $<$ 1.0) dramatically reduces refusal rate, 
from 96.2\% at natural activation to 48.7\% when completely 
zeroed out. This nearly 50 percentage point drop demonstrates 
that the safety circuit is necessary for refusal behavior. 
Amplifying activations (scale $>$ 1.0) produces modest further 
increases, reaching 99.4\% at scale=10, suggesting the circuit 
is already operating near saturation under natural conditions.
As a control, we perform the same intervention on randomly 
selected features of equal size. The random circuit shows no 
systematic response to scaling: refusal rate fluctuates within 
a narrow band ($\pm$8\%) around the mean, with no consistent 
trend as scale increases. This confirms that the observed 
causal effect is specific to the identified safety circuit, 
not a general artifact of feature intervention.
The consistent pattern across Llama, Qwen, and Gemma confirms that the identified safety circuits causally mediate safety behavior.
These findings validate our circuit identification methodology 
and justify using the safety circuit as an anchor during 
evolution. The circuit is both \textit{necessary} (suppression 
causes safety collapse) and \textit{specific} (random features 
have no effect), analogous to how Hox genes are both essential 
for body plan and distinct from peripheral genes.

\begin{figure*}[ht]
  \vspace{-15pt}
  \begin{center}
    \centerline{\includegraphics[width=\textwidth]{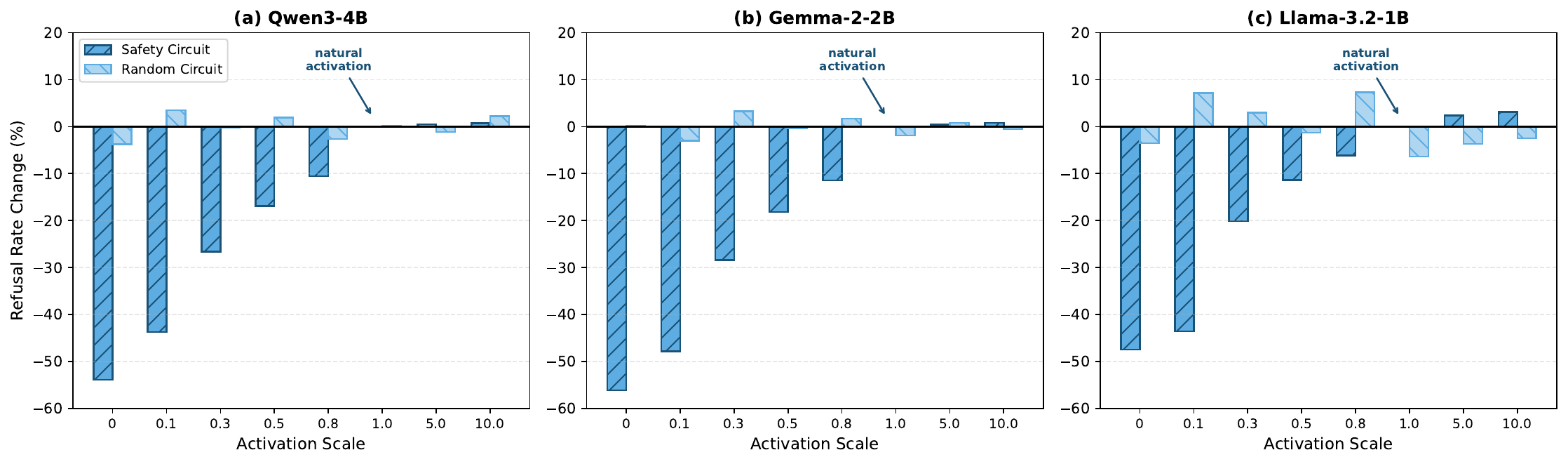}}
    \vspace{-10pt}
    \caption{
Causal verification of the safety circuit across three model families. We intervene on 
feature activations by scaling them from 0 to 10$\times$ and 
measure the resulting change in refusal rate. \textbf{Safety Circuit} (blue): scaling down suppresses refusal behavior dramatically, while scaling up slightly amplifies it. \textbf{Random Circuit} (light blue): scaling 
has no systematic effect, with refusal rate fluctuating within 
$\pm$8\% around the mean, confirming that the effect is specific 
to the identified safety circuit rather than a general property 
of feature intervention.
    }
    \label{fig:circuit_validity}
  \end{center}
  \vspace{-25pt}
\end{figure*}

\subsection{Safety Preservation During Evolution}

Having established that the identified safety circuit causally 
mediates refusal behavior, we now examine whether anchoring this 
circuit can preserve safety during self-evolution.
We compare three approaches: (1) unconstrained evolution without 
any safety mechanism, (2) explicit constraint using Beaver-Cost~\citep{dai2023} 
as a safety reward model that penalizes unsafe outputs during 
evolution, and (3) our implicit constraint (CAE) with KL divergence 
on safety circuits. For CAE, we set constraint weight $\lambda = 0.1$ 
and use $M = 64$ reference samples for computing the circuit KL 
loss. Safety circuits are extracted once before evolution using the circuit-tracer toolkit, comprising 1.2\%--1.8\% of total transcoder features across different models.
Figure~\ref{fig:pareto} traces the safety-capability trajectories of three models under different constraint strategies. Each trajectory begins at the original aligned model and progresses through evolution.
Without any constraint, all models follow a troubling pattern: capability improves steadily while safety collapses. For the EVOL-RL evolution algorithm, the safety score of Qwen3-4B-Instruct drops from $99.04$\% to $69.85$\%; that of Gemma-2-2B-IT drops from $98.65$\% to $75.38$\%; 
Llama-3.2-1B-Instruct from $96.22$\% to $72.13$\%.
The same phenomenon is also observed for the Abs-ZERO evolution algorithm.
The models become stronger reasoners but lose their ability to refuse harmful requests.
Adding an explicit safety reward model partially arrests this decline. 
Although the decline in safety performance has diminished, the increase in capabilities during the evolutionary process has also diminished. 
The reward signal, while protective, appears to interfere with task learning.
Circuit anchoring produces a qualitatively different trajectory. 
Safety remains above $95$\% throughout evolution, while capability gains match those of unconstrained evolution. The models evolve freely in capability-relevant dimensions while remaining anchored in safety-critical ones. This pattern holds across all three model families and both evolution algorithms. Furthermore, to verify the scalability and durability of our approach, we conduct extended stress-test experiments on larger models (Gemma-2-9B-IT) and significantly longer evolution horizons (up to 5,000 steps). 
As detailed in Appendix~\ref{app:scalability}, 
\begin{wrapfigure}{r}{0.45\textwidth}
\vspace{-5pt}
\centering
\includegraphics[width=\linewidth]{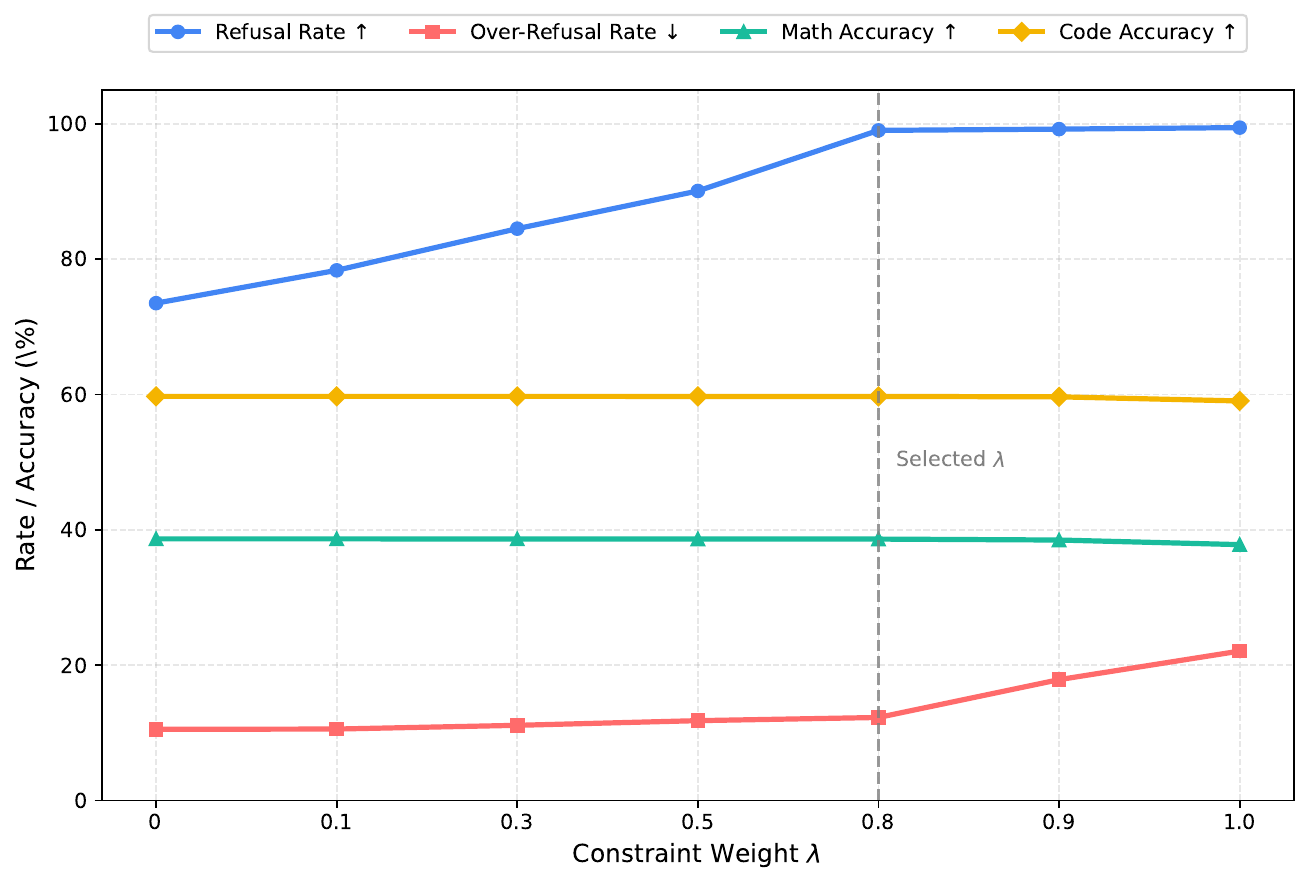}
\vspace{-15pt}
\caption{Effect of constraint weight $\lambda$ of Llama-3.2-1B-Instruct with EVOL-RL. Small $\lambda$ provides insufficient safety, while excessively large values cause capability degradation and increased over-refusal.}
\label{fig:ablation_lambda}
\vspace{-40pt}
\end{wrapfigure}CAE consistently maintains high safety ($>95\%$) without compromising capability gains, demonstrating its robust protection even under extreme evolutionary pressure.

\begin{figure*}[t]
\vspace{-20pt}
\centering
\includegraphics[width=1\textwidth]{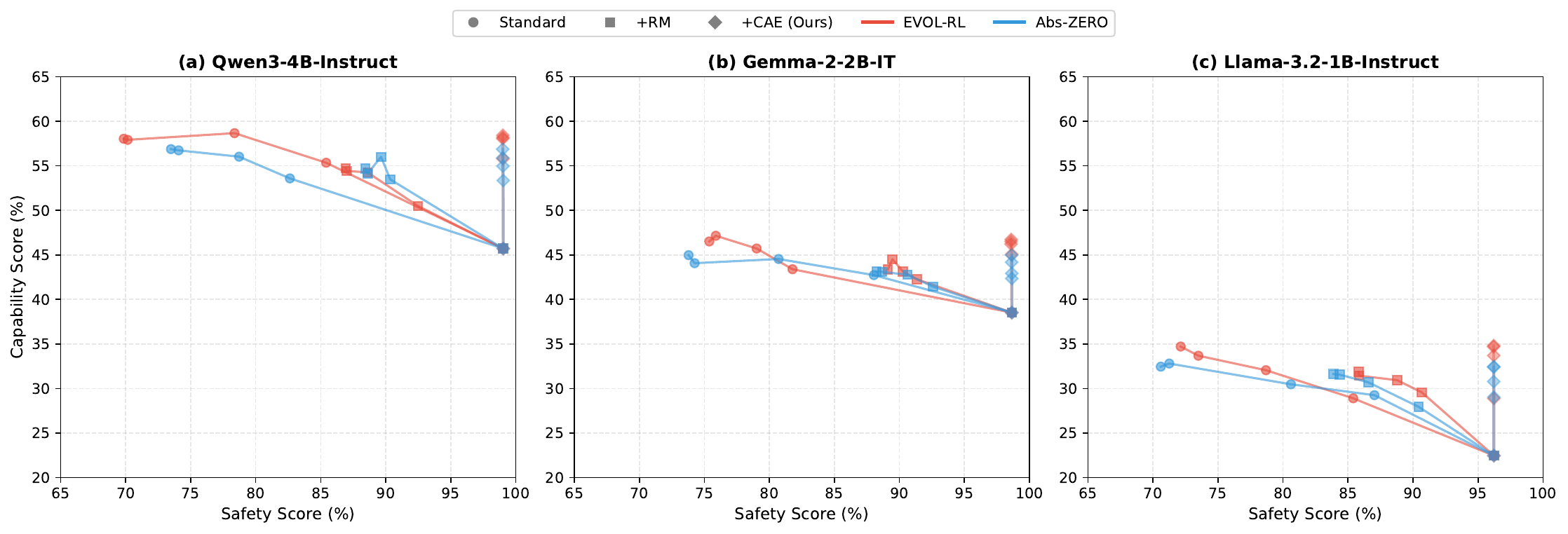}
\vspace{-15pt}
\caption{Safety-capability trajectories during evolution across 
three model families. 
Standard evolution improves 
capability but severely degrades safety. +RM partially preserves 
safety at the cost of reduced capability gains. +CAE maintains 
high safety while achieving comparable capability improvement to unconstrained evolution. The CAE trajectories 
cluster tightly in the upper-right region across both algorithms, 
demonstrating consistent safety preservation regardless of the 
underlying evolution method.}
\label{fig:pareto}
\vspace{-10pt}
\end{figure*}

\subsection{Effect of Constraint Weight}
The hyperparameter $\lambda$ controls the strength of the circuit 
anchoring constraint.
Figure~\ref{fig:ablation_lambda} shows the effect of varying $\lambda$ across safety and capability metrics. When $\lambda = 0$ (no constraint), the model achieves high 
capability scores but suffers severe safety degradation. As $\lambda$ increases, refusal rate improves 
substantially, reaching around $90$\% at $\lambda = 0.5$.
However, excessively large $\lambda$ values introduce two negative 
effects. First, capability scores begin to decline (although slightly). Second, over-refusal rate begins to rise sharply, indicating the model 
becomes overly conservative and refuses benign requests.
We use $0.8$ as the constraint weight for the safety circuit in evolution, achieving a high refusal rate ($>95$\%) while maintaining capability and keeping over-refusal below $18$\%.

\section{Related Work}
Self-evolution enhances model capabilities without extensive human supervision via iterative policy optimization \cite{evol-rl} or abstraction-based reasoning \cite{abs0}. However, these methods uniformly neglect safety preservation. Given mounting evidence that safety alignment is surprisingly brittle \cite{ji2025languagemodelsresistalignment,qi2024safetyalignmentjusttokens} and easily compromised by benign fine-tuning \citep{qi2023finetuningalignedlanguagemodels, rebuttal2}, this neglect poses severe risks. Recent analyses reveal that such safety collapse is closely tied to the similarity between alignment and fine-tuning data \cite{rebuttal3}. While some explicit constraint methods have been proposed to restore safety during fine-tuning \cite{rebuttal1}, they rely on external supervision and are difficult to scale. This motivates our search for structural, rather than purely behavioral, safety preservation in the self-evolution setting.
Our approach builds on mechanistic interpretability \cite{smart2025beyond}, specifically the circuits framework \cite{tigges2024llmcircuitanalysesconsistent} and transcoders \cite{transcoder}, which decompose activations into interpretable features. Recent work localizes refusal behaviors to specific directions \cite{arditi2024refusal} and traces information flow \cite{birardi2025automated}. We leverage the circuit-tracer toolkit \citep{circuit-tracer} to identify safety circuits and integrate them as optimization constraints.
Biologically, our approach is inspired by evolutionary developmental biology, where master regulatory genes (e.g., Hox genes) remain highly conserved across millions of years \cite{carroll2005endless,mcginnis1994century}. This conservation arises from purifying selection \cite{stern2000perspective}, enforcing developmental constraints \cite{maynard1985developmental} that forbid fatal evolutionary paths. Coupled with modularity \cite{wagner2007road,felix2015pervasive}, this allows robustness and evolvability to coexist. We propose aligned language models exhibit analogous modularity, where safety circuits must remain conserved while capability components evolve. Finally, unlike standard constrained optimization (e.g., PPO \citep{ppo} or elastic weight consolidation \cite{li2020few}), our approach operates in the disentangled feature space, enabling surgical preservation of safety without impeding capability.

\section{Conclusion}
Current self-evolution algorithms optimize purely for capability, leading to misevolution where models gain reasoning but progressively lose safety. Mirroring biological evolution, where conserved Hox genes anchor essential functions while peripheral genes adapt freely, we propose Circuit-Anchored Evolution (CAE). By identifying a minimal safety circuit via mechanistic interpretability and anchoring its activations within a small displacement bound, CAE acts as an artificial purifying selection. It prevents safety-disrupting updates while allowing free capability adaptation. Experiments across diverse models and algorithms demonstrate that CAE achieves superior safety preservation, matches unconstrained capability gains, and incurs significantly lower overhead than reward-based alternatives. Ultimately, CAE enables AI systems to grow stronger without growing dangerous, achieving the evolutionary balance that allows life to flourish.

{\small
\bibliographystyle{plainnat}
\bibliography{refs}

@misc{misevolve,
      title={Your Agent May Misevolve: Emergent Risks in Self-evolving LLM Agents}, 
      author={Shuai Shao and Qihan Ren and Chen Qian and Boyi Wei and Dadi Guo and Jingyi Yang and Xinhao Song and Linfeng Zhang and Weinan Zhang and Dongrui Liu and Jing Shao},
      year={2025},
      eprint={2509.26354},
      archivePrefix={arXiv},
      primaryClass={cs.AI},
      url={https://arxiv.org/abs/2509.26354}, 
}

@misc{evol-rl,
      title={Evolving Language Models without Labels: Majority Drives Selection, Novelty Promotes Variation}, 
      author={Yujun Zhou and Zhenwen Liang and Haolin Liu and Wenhao Yu and Kishan Panaganti and Linfeng Song and Dian Yu and Xiangliang Zhang and Haitao Mi and Dong Yu},
      year={2025},
      eprint={2509.15194},
      archivePrefix={arXiv},
      primaryClass={cs.LG},
      url={https://arxiv.org/abs/2509.15194}, 
}

@misc{abs0,
      title={Absolute Zero: Reinforced Self-play Reasoning with Zero Data}, 
      author={Andrew Zhao and Yiran Wu and Yang Yue and Tong Wu and Quentin Xu and Yang Yue and Matthieu Lin and Shenzhi Wang and Qingyun Wu and Zilong Zheng and Gao Huang},
      year={2025},
      eprint={2505.03335},
      archivePrefix={arXiv},
      primaryClass={cs.LG},
      url={https://arxiv.org/abs/2505.03335}, 
}

@misc{dr0,
      title={Dr. Zero: Self-Evolving Search Agents without Training Data}, 
      author={Zhenrui Yue and Kartikeya Upasani and Xianjun Yang and Suyu Ge and Shaoliang Nie and Yuning Mao and Zhe Liu and Dong Wang},
      year={2026},
      eprint={2601.07055},
      archivePrefix={arXiv},
      primaryClass={cs.AI},
      url={https://arxiv.org/abs/2601.07055}, 
}

@misc{self-evol1,
      title={UI-Genie: A Self-Improving Approach for Iteratively Boosting MLLM-based Mobile GUI Agents}, 
      author={Han Xiao and Guozhi Wang and Yuxiang Chai and Zimu Lu and Weifeng Lin and Hao He and Lue Fan and Liuyang Bian and Rui Hu and Liang Liu and Shuai Ren and Yafei Wen and Xiaoxin Chen and Aojun Zhou and Hongsheng Li},
      year={2025},
      eprint={2505.21496},
      archivePrefix={arXiv},
      primaryClass={cs.CL},
      url={https://arxiv.org/abs/2505.21496}, 
}

@article{miconi2008evolution,
  title={Evolution and complexity: The double-edged sword},
  author={Miconi, Thomas},
  journal={Artificial life},
  volume={14},
  number={3},
  pages={325--344},
  year={2008},
  publisher={MIT Press One Rogers Street, Cambridge, MA 02142-1209, USA journals-info~…}
}

@article{hanley2011double,
  title={The double-edged sword: How evolution can make or break a live-attenuated virus vaccine},
  author={Hanley, Kathryn A},
  journal={Evolution: Education and Outreach},
  volume={4},
  number={4},
  pages={635--643},
  year={2011},
  publisher={Springer}
}

@article{perc2015double,
  title={A double-edged sword: Benefits and pitfalls of heterogeneous punishment in evolutionary inspection games},
  author={Perc, Matja{\v{z}} and Szolnoki, Attila},
  journal={Scientific reports},
  volume={5},
  number={1},
  pages={11027},
  year={2015},
  publisher={Nature Publishing Group UK London}
}

@article{emery2005evolution,
  title={Evolution of the avian brain and intelligence},
  author={Emery, Nathan J and Clayton, Nicola S},
  journal={Current Biology},
  volume={15},
  number={23},
  pages={R946--R950},
  year={2005},
  publisher={Elsevier}
}

@article{williams2023eagle,
  title={Eagle eyed or bird brained?},
  author={Williams, David},
  journal={Eye},
  volume={37},
  number={12},
  pages={2426--2430},
  year={2023},
  publisher={Nature Publishing Group UK London}
}

@article{alberch1989logic,
  title={The logic of monsters: evidence for internal constraint in development and evolution},
  author={Alberch, Pere},
  journal={Geobios},
  volume={22},
  pages={21--57},
  year={1989},
  publisher={Elsevier}
}

@article{mallo2010hox,
  title={Hox genes and regional patterning of the vertebrate body plan},
  author={Mallo, Moises and Wellik, Deneen M and Deschamps, Jacqueline},
  journal={Developmental biology},
  volume={344},
  number={1},
  pages={7--15},
  year={2010},
  publisher={Elsevier}
}

@article{hughes2002hox,
  title={Hox genes and the evolution of the arthropod body plan 1},
  author={Hughes, Cynthia L and Kaufman, Thomas C},
  journal={Evolution \& development},
  volume={4},
  number={6},
  pages={459--499},
  year={2002},
  publisher={Wiley Online Library}
}

@article{pearson2005modulating,
  title={Modulating Hox gene functions during animal body patterning},
  author={Pearson, Joseph C and Lemons, Derek and McGinnis, William},
  journal={Nature Reviews Genetics},
  volume={6},
  number={12},
  pages={893--904},
  year={2005},
  publisher={Nature Publishing Group UK London}
}

@article{goodman2001human,
  title={Human HOX gene mutations},
  author={Goodman, FR and Scambler, PJ},
  journal={Clinical genetics},
  volume={59},
  number={1},
  pages={1--11},
  year={2001},
  publisher={Wiley Online Library}
}

@article{brunet2021role,
  title={The role of purifying selection in the origin and maintenance of complex function},
  author={Brunet, Tyler DP and Doolittle, W Ford and Bielawski, Joseph P},
  journal={Studies in History and Philosophy of Science Part A},
  volume={87},
  pages={125--135},
  year={2021},
  publisher={Elsevier}
}

@article{sharov2014evolutionary,
  title={Evolutionary constraints or opportunities?},
  author={Sharov, Alexei A},
  journal={Biosystems},
  volume={123},
  pages={9--18},
  year={2014},
  publisher={Elsevier}
}

@misc{tao2024surveyselfevolutionlargelanguage,
      title={A Survey on Self-Evolution of Large Language Models}, 
      author={Zhengwei Tao and Ting-En Lin and Xiancai Chen and Hangyu Li and Yuchuan Wu and Yongbin Li and Zhi Jin and Fei Huang and Dacheng Tao and Jingren Zhou},
      year={2024},
      eprint={2404.14387},
      archivePrefix={arXiv},
      primaryClass={cs.CL},
      url={https://arxiv.org/abs/2404.14387}, 
}

@article{gao2025survey,
  title={A survey of self-evolving agents: On path to artificial super intelligence},
  author={Gao, Huan-ang and Geng, Jiayi and Hua, Wenyue and Hu, Mengkang and Juan, Xinzhe and Liu, Hongzhang and Liu, Shilong and Qiu, Jiahao and Qi, Xuan and Wu, Yiran and others},
  journal={arXiv preprint arXiv:2507.21046},
  year={2025}
}

@article{feng2025group,
  title={Group-in-group policy optimization for llm agent training},
  author={Feng, Lang and Xue, Zhenghai and Liu, Tingcong and An, Bo},
  journal={arXiv preprint arXiv:2505.10978},
  year={2025}
}

@article{novikov2025alphaevolve,
  title={AlphaEvolve: A coding agent for scientific and algorithmic discovery},
  author={Novikov, Alexander and V{\~u}, Ng{\^a}n and Eisenberger, Marvin and Dupont, Emilien and Huang, Po-Sen and Wagner, Adam Zsolt and Shirobokov, Sergey and Kozlovskii, Borislav and Ruiz, Francisco JR and Mehrabian, Abbas and others},
  journal={arXiv preprint arXiv:2506.13131},
  year={2025}
}

@article{transcoder,
  title={Transcoders find interpretable llm feature circuits},
  author={Dunefsky, Jacob and Chlenski, Philippe and Nanda, Neel},
  journal={Advances in Neural Information Processing Systems},
  volume={37},
  pages={24375--24410},
  year={2024}
}

@misc{ji2025languagemodelsresistalignment,
      title={Language Models Resist Alignment: Evidence From Data Compression}, 
      author={Jiaming Ji and Kaile Wang and Tianyi Qiu and Boyuan Chen and Jiayi Zhou and Changye Li and Hantao Lou and Juntao Dai and Yunhuai Liu and Yaodong Yang},
      year={2025},
      eprint={2406.06144},
      archivePrefix={arXiv},
      primaryClass={cs.CL},
      url={https://arxiv.org/abs/2406.06144}, 
}

@misc{qi2024safetyalignmentjusttokens,
      title={Safety Alignment Should Be Made More Than Just a Few Tokens Deep}, 
      author={Xiangyu Qi and Ashwinee Panda and Kaifeng Lyu and Xiao Ma and Subhrajit Roy and Ahmad Beirami and Prateek Mittal and Peter Henderson},
      year={2024},
      eprint={2406.05946},
      archivePrefix={arXiv},
      primaryClass={cs.CR},
      url={https://arxiv.org/abs/2406.05946}, 
}

@misc{qi2023finetuningalignedlanguagemodels,
      title={Fine-tuning Aligned Language Models Compromises Safety, Even When Users Do Not Intend To!}, 
      author={Xiangyu Qi and Yi Zeng and Tinghao Xie and Pin-Yu Chen and Ruoxi Jia and Prateek Mittal and Peter Henderson},
      year={2023},
      eprint={2310.03693},
      archivePrefix={arXiv},
      primaryClass={cs.CL},
      url={https://arxiv.org/abs/2310.03693}, 
}

@article{smart2025beyond,
  title={Beyond model interpretability: Socio-structural explanations in machine learning},
  author={Smart, Andrew and Kasirzadeh, Atoosa},
  journal={AI \& SOCIETY},
  volume={40},
  number={4},
  pages={2045--2053},
  year={2025},
  publisher={Springer}
}

@misc{circuit-tracer,
  author = {Hanna, Michael and Piotrowski, Mateusz and Lindsey, Jack and Ameisen, Emmanuel},
  title = {circuit-tracer},
  howpublished = {\url{https://github.com/safety-research/circuit-tracer}},
  note = {The first two authors contributed equally and are listed alphabetically.},
  year = {2025}
}

@article{arditi2024refusal,
  title={Refusal in language models is mediated by a single direction},
  author={Arditi, Andy and Obeso, Oscar and Syed, Aaquib and Paleka, Daniel and Panickssery, Nina and Gurnee, Wes and Nanda, Neel},
  journal={Advances in Neural Information Processing Systems},
  volume={37},
  pages={136037--136083},
  year={2024}
}

@article{birardi2025automated,
  title={Automated Circuit Interpretation via Probe Prompting},
  author={Birardi, Giuseppe},
  journal={arXiv preprint arXiv:2511.07002},
  year={2025}
}

@misc{ppo,
      title={Proximal Policy Optimization Algorithms}, 
      author={John Schulman and Filip Wolski and Prafulla Dhariwal and Alec Radford and Oleg Klimov},
      year={2017},
      eprint={1707.06347},
      archivePrefix={arXiv},
      primaryClass={cs.LG},
      url={https://arxiv.org/abs/1707.06347}, 
}

@article{li2020few,
  title={Few-shot image generation with elastic weight consolidation},
  author={Li, Yijun and Zhang, Richard and Lu, Jingwan and Shechtman, Eli},
  journal={arXiv preprint arXiv:2012.02780},
  year={2020}
}

@book{carroll2005endless,
  title={Endless forms most beautiful: The new science of evo devo and the making of the animal kingdom},
  author={Carroll, Sean B},
  number={54},
  year={2005},
  publisher={WW Norton \& Company}
}

@article{mcginnis1994century,
  title={A century of homeosis, a decade of homeoboxes},
  author={McGinnis, William},
  journal={Genetics},
  volume={137},
  number={3},
  pages={607},
  year={1994}
}

@article{stern2000perspective,
  title={Perspective: evolutionary developmental biology and the problem of variation},
  author={Stern, David L},
  journal={Evolution},
  volume={54},
  number={4},
  pages={1079--1091},
  year={2000},
  publisher={Wiley Online Library}
}

@article{maynard1985developmental,
  title={Developmental constraints and evolution},
  author={Maynard Smith, John and Burian, Richard and Kauffman, Stuart and Alberch, Pere and Campbell, John and Goodwin, Brian and Lande, Russell and Raup, David and Wolpert, Lewis},
  journal={Quarterly Review of Biology},
  volume={60},
  number={3},
  pages={265--287},
  year={1985}
}

@article{wagner2007road,
  title={The road to modularity},
  author={Wagner, G{\"u}nter P and Pavlicev, Mihaela and Cheverud, James M},
  journal={Nature Reviews Genetics},
  volume={8},
  number={12},
  pages={921--931},
  year={2007},
  publisher={Nature Publishing Group UK London}
}

@article{felix2015pervasive,
  title={Pervasive robustness in biological systems},
  author={F{\'e}lix, Marie-Anne and Barkoulas, Michalis},
  journal={Nature Reviews Genetics},
  volume={16},
  number={8},
  pages={483--496},
  year={2015},
  publisher={Nature Publishing Group UK London}
}

@article{shao2024deepseekmath,
  title={Deepseekmath: Pushing the limits of mathematical reasoning in open language models},
  author={Shao, Zhihong and Wang, Peiyi and Zhu, Qihao and Xu, Runxin and Song, Junxiao and Bi, Xiao and Zhang, Haowei and Zhang, Mingchuan and Li, YK and Wu, Yang and others},
  journal={arXiv preprint arXiv:2402.03300},
  year={2024}
}

@inproceedings{hanna2025circuit,
  title={Circuit-tracer: A new library for finding feature circuits},
  author={Hanna, Michael and Piotrowski, Mateusz and Lindsey, Jack and Ameisen, Emmanuel},
  booktitle={Proceedings of the 8th BlackboxNLP Workshop: Analyzing and Interpreting Neural Networks for NLP},
  pages={239--249},
  year={2025}
}

@misc{dai2023,
      title={Safe RLHF: Safe Reinforcement Learning from Human Feedback}, 
      author={Josef Dai and Xuehai Pan and Ruiyang Sun and Jiaming Ji and Xinbo Xu and Mickel Liu and Yizhou Wang and Yaodong Yang},
      year={2023},
      eprint={2310.12773},
      archivePrefix={arXiv},
      primaryClass={cs.AI},
      url={https://arxiv.org/abs/2310.12773}, 
}

@article{zou2023universal,
  title={Universal and transferable adversarial attacks on aligned language models},
  author={Zou, Andy and Wang, Zifan and Carlini, Nicholas and Nasr, Milad and Kolter, J Zico and Fredrikson, Matt},
  journal={arXiv preprint arXiv:2307.15043},
  year={2023}
}

@article{li2024safety,
  title={Safety layers in aligned large language models: The key to llm security},
  author={Li, Shen and Yao, Liuyi and Zhang, Lan and Li, Yaliang},
  journal={arXiv preprint arXiv:2408.17003},
  year={2024}
}

@misc{tigges2024llmcircuitanalysesconsistent,
      title={LLM Circuit Analyses Are Consistent Across Training and Scale}, 
      author={Curt Tigges and Michael Hanna and Qinan Yu and Stella Biderman},
      year={2024},
      eprint={2407.10827},
      archivePrefix={arXiv},
      primaryClass={cs.LG},
      url={https://arxiv.org/abs/2407.10827}, 
}

@misc{Do-Not-Answer,
      title={Do-Not-Answer: A Dataset for Evaluating Safeguards in LLMs}, 
      author={Yuxia Wang and Haonan Li and Xudong Han and Preslav Nakov and Timothy Baldwin},
      year={2023},
      eprint={2308.13387},
      archivePrefix={arXiv},
      primaryClass={cs.CL},
      url={https://arxiv.org/abs/2308.13387}, 
}

@misc{DAN,
      title={"Do Anything Now": Characterizing and Evaluating In-The-Wild Jailbreak Prompts on Large Language Models}, 
      author={Xinyue Shen and Zeyuan Chen and Michael Backes and Yun Shen and Yang Zhang},
      year={2024},
      eprint={2308.03825},
      archivePrefix={arXiv},
      primaryClass={cs.CR},
      url={https://arxiv.org/abs/2308.03825}, 
}

@misc{MaliciousInstruct,
      title={Catastrophic Jailbreak of Open-source LLMs via Exploiting Generation}, 
      author={Yangsibo Huang and Samyak Gupta and Mengzhou Xia and Kai Li and Danqi Chen},
      year={2023},
      eprint={2310.06987},
      archivePrefix={arXiv},
      primaryClass={cs.CL},
      url={https://arxiv.org/abs/2310.06987}, 
}

@misc{rebuttal1,
      title={Shape it Up! Restoring LLM Safety during Finetuning}, 
      author={ShengYun Peng and Pin-Yu Chen and Jianfeng Chi and Seongmin Lee and Duen Horng Chau},
      year={2025},
      eprint={2505.17196},
      archivePrefix={arXiv},
      primaryClass={cs.LG},
      url={https://arxiv.org/abs/2505.17196}, 
}

@misc{rebuttal2,
      title={Fine-Tuning Lowers Safety and Disrupts Evaluation Consistency}, 
      author={Kathleen C. Fraser and Hillary Dawkins and Isar Nejadgholi and Svetlana Kiritchenko},
      year={2025},
      eprint={2506.17209},
      archivePrefix={arXiv},
      primaryClass={cs.CL},
      url={https://arxiv.org/abs/2506.17209}, 
}

@misc{rebuttal3,
      title={Why LLM Safety Guardrails Collapse After Fine-tuning: A Similarity Analysis Between Alignment and Fine-tuning Datasets}, 
      author={Lei Hsiung and Tianyu Pang and Yung-Chen Tang and Linyue Song and Tsung-Yi Ho and Pin-Yu Chen and Yaoqing Yang},
      year={2025},
      eprint={2506.05346},
      archivePrefix={arXiv},
      primaryClass={cs.CR},
      url={https://arxiv.org/abs/2506.05346}, 
}
}


\appendix

\section{Notation Summary}
\label{app:notation}

For convenience, we summarize the notation used throughout the paper in Table~\ref{tab:notation}.

\begin{table}[h]
\centering
\caption{Summary of notation.}
\label{tab:notation}
\begin{tabular}{cl}
\toprule
\textbf{Symbol} & \textbf{Description} \\
\midrule
$\mathcal{M}_\theta$ & Language model with parameters $\theta$ \\
$L$ & Number of layers in the model \\
$d$ & Hidden dimension \\
$K$ & Number of transcoder features per layer \\
$h_l(x; \theta)$ & Hidden state at layer $l$ for input $x$ \\
$\mathcal{T}_l = (E_l, D_l)$ & Transcoder (encoder-decoder) for layer $l$ \\
$f_l(x; \theta)$ & Feature activation vector at layer $l$ \\
$f^{(l,k)}(x; \theta)$ & Activation of feature $k$ at layer $l$ \\
$\mathcal{S}$ & Safety circuit (set of feature indices) \\
$\mathcal{C}$ & Capability circuit (set of feature indices) \\
$f_{\mathcal{S}}(x; \theta)$ & Safety feature activations \\
$\alpha^{(l,k)}_y$ & Attribution score of feature $(l,k)$ to output $y$ \\
$\pi_\theta(y\vert{}x)$ & Output probability distribution \\
$P_\theta^{\mathcal{S}}$ & Distribution of safety circuit activations \\
$\mathcal{L}_{\text{evol}}$ & Evolution objective (GRPO or Abs-ZERO) \\
$\mathcal{L}_{\text{circuit}}$ & Circuit-level KL constraint \\
$\mathcal{L}_{\text{CAE}}$ & Circuit-Anchored Evolution objective \\
$\lambda$ & Constraint weight hyperparameter \\
$\theta_0$ & Initial aligned model parameters \\
$\theta^{(t)}$ & Model parameters at iteration $t$ \\
\bottomrule
\end{tabular}
\end{table}

\section{Proof of Proposition 4.6 (Gradient of Circuit Loss)}
\label{app:gradient_proof}

\begin{proposition*}[Restated]
Under Assumption~\ref{ass:regularity}, the gradient of the circuit loss admits the form:
\begin{equation}
    \nabla_\theta \mathcal{L}_{\text{circuit}}(\theta) = \mathbb{E}_{x \sim p(x)} \left[ \sum_{(l,k) \in \mathcal{S}} \left( 1 + \log \frac{f^{(l,k)}(x; \theta)}{f^{(l,k)}(x; \theta_0)} \right) \nabla_\theta f^{(l,k)}(x; \theta) \right]
\end{equation}
\end{proposition*}

\begin{proof}
We begin by expanding the circuit loss. By Definition~\ref{def:cao}:
\begin{equation}
    \mathcal{L}_{\text{circuit}}(\theta) = \mathbb{E}_{x \sim p(x)} \left[ \sum_{(l,k) \in \mathcal{S}} D_{KL}\left( f^{(l,k)}(x; \theta_0) \| f^{(l,k)}(x; \theta) \right) \right]
\end{equation}

For notational simplicity, we consider a single feature $(l,k)$ and drop the expectation temporarily. Let $p := f^{(l,k)}(x; \theta_0)$ (reference, fixed) and $q := f^{(l,k)}(x; \theta)$ (current, variable).

Since we treat feature activations as parameters of distributions (specifically, we use a softmax normalization over the feature dimension to obtain valid probability distributions), the KL divergence is:
\begin{equation}
    D_{KL}(p \| q) = \sum_i p_i \log \frac{p_i}{q_i} = \sum_i p_i \log p_i - \sum_i p_i \log q_i
\end{equation}

The first term is constant with respect to $\theta$. Taking the gradient of the second term:
\begin{align}
    \nabla_\theta D_{KL}(p \| q) &= -\nabla_\theta \sum_i p_i \log q_i \\
    &= -\sum_i p_i \cdot \frac{1}{q_i} \cdot \nabla_\theta q_i \\
    &= -\sum_i \frac{p_i}{q_i} \nabla_\theta q_i
\end{align}

Now, we need to compute $\nabla_\theta q_i = \nabla_\theta f^{(l,k)}(x; \theta)$. By the chain rule:
\begin{equation}
    \nabla_\theta f^{(l,k)}(x; \theta) = \nabla_\theta E_l(h_l(x; \theta))_k = J_{E_l}^{(k)} \cdot \nabla_\theta h_l(x; \theta)
\end{equation}
where $J_{E_l}^{(k)}$ is the $k$-th row of the Jacobian of $E_l$, which is fixed since the transcoder is frozen.

For the case where we treat activations directly (without softmax normalization), we use a Gaussian approximation. Assuming $f^{(l,k)} \sim \mathcal{N}(\mu, \sigma^2)$, the KL divergence between two Gaussians with the same variance is:
\begin{equation}
    D_{KL}(\mathcal{N}(\mu_0, \sigma^2) \| \mathcal{N}(\mu, \sigma^2)) = \frac{(\mu - \mu_0)^2}{2\sigma^2}
\end{equation}

Taking the gradient:
\begin{equation}
    \nabla_\theta D_{KL} = \frac{\mu - \mu_0}{\sigma^2} \nabla_\theta \mu = \frac{f^{(l,k)}(x; \theta) - f^{(l,k)}(x; \theta_0)}{\sigma^2} \nabla_\theta f^{(l,k)}(x; \theta)
\end{equation}

More generally, using the score function identity for exponential families, we can write:
\begin{equation}
    \nabla_\theta D_{KL}(p \| q) = \mathbb{E}_{p}\left[ \nabla_\theta \log q \right] = \left(1 + \log \frac{q}{p}\right) \nabla_\theta q
\end{equation}

Summing over all features in $\mathcal{S}$ and taking expectations completes the proof.
\end{proof}

\section{Proof of Theorem~\ref{thm:safety_bound} (Safety Bound)}
\label{app:safety_proof}

\begin{theorem*}[Restated] 
Let $\theta_0$ be an aligned model and $\theta^{(T)}$ be the model after $T$ steps of CAE training. Under Assumption~\ref{ass:regularity} (Conditions 1 and 2), if $\mathcal{L}_{\text{anchor}}(\theta^{(T)}) \leq \epsilon$, then: 
\begin{equation} 
    \left| \mathbb{E}_{x \sim \mathcal{D}_{\text{harm}}} \left[ \pi_{\theta^{(T)}}(y_{\text{refuse}} | x) - \pi_{\theta_0}(y_{\text{refuse}} | x) \right] \right| \leq C \sqrt{\epsilon} 
\end{equation} 
where $C$ depends on $B$, $L_f$, and $\vert{}\mathcal{S}\vert{}$. Note that this proof does not require the transcoder to be differentiable (Condition 3).
\end{theorem*}

\begin{proof}
The proof proceeds in three steps.

\textbf{Step 1: From KL to Total Variation.}

By Pinsker's inequality, for any two distributions $P$ and $Q$:
\begin{equation}
    \|P - Q\|_{\text{TV}} \leq \sqrt{\frac{1}{2} D_{KL}(P \| Q)}
\end{equation}

Applying this to the circuit activation distributions:
\begin{equation}
    \|P_{\theta_0}^{\mathcal{S}} - P_{\theta^{(T)}}^{\mathcal{S}}\|_{\text{TV}} \leq \sqrt{\frac{1}{2} D_{KL}(P_{\theta_0}^{\mathcal{S}} \| P_{\theta^{(T)}}^{\mathcal{S}})} \leq \sqrt{\frac{\epsilon}{2}}
\end{equation}

\textbf{Step 2: Lipschitz Property of Refusal Probability.}

We now establish that the refusal probability is Lipschitz in the safety circuit activations. By Definition~\ref{def:safety_circuit}, the safety circuit $\mathcal{S}$ consists of features with high attribution scores for the refusal output. This means:
\begin{equation}
    \pi_\theta(y_{\text{refuse}} | x) \approx g(f_{\mathcal{S}}(x; \theta))
\end{equation}
for some function $g$ that depends primarily on the safety features.

By Assumption~\ref{ass:regularity}(2), the feature map is $L_f$-Lipschitz. Furthermore, the output probability is a softmax function, which is Lipschitz with constant at most 1 in its inputs. Therefore, there exists a constant $L_g$ such that:
\begin{equation}
    |\pi_\theta(y_{\text{refuse}} | x) - \pi_{\theta'}(y_{\text{refuse}} | x)| \leq L_g \|f_{\mathcal{S}}(x; \theta) - f_{\mathcal{S}}(x; \theta')\|
\end{equation}

\textbf{Step 3: Combining the Bounds.}

By the definition of total variation distance:
\begin{align}
    &\left| \mathbb{E}_{x} \left[ \pi_{\theta^{(T)}}(y_{\text{refuse}} | x) \right] - \mathbb{E}_{x} \left[ \pi_{\theta_0}(y_{\text{refuse}} | x) \right] \right| \\
    &\leq \mathbb{E}_{x} \left[ \left| \pi_{\theta^{(T)}}(y_{\text{refuse}} | x) - \pi_{\theta_0}(y_{\text{refuse}} | x) \right| \right] \\
    &\leq L_g \mathbb{E}_{x} \left[ \|f_{\mathcal{S}}(x; \theta^{(T)}) - f_{\mathcal{S}}(x; \theta_0)\| \right] \\
    &\leq L_g \sqrt{|\mathcal{S}|} \cdot \|P_{\theta_0}^{\mathcal{S}} - P_{\theta^{(T)}}^{\mathcal{S}}\|_{\text{TV}} \cdot B \\
    &\leq L_g \sqrt{|\mathcal{S}|} \cdot B \cdot \sqrt{\frac{\epsilon}{2}}
\end{align}

where the third inequality uses Cauchy-Schwarz and the bounded activation assumption.

Setting $C = L_g \sqrt{\vert{}\mathcal{S}\vert{}} \cdot B / \sqrt{2}$ completes the proof.
\end{proof}

\section{Additional Theoretical Results}
\label{app:additional_theory}

\subsection{Convergence Analysis}

We provide a convergence guarantee for the CAE algorithm under standard assumptions.

\begin{assumption}[Smoothness]
\label{ass:smoothness}
The losses $\mathcal{L}_{\text{evol}}$ and $\mathcal{L}_{\text{circuit}}$ are $\beta$-smooth, i.e., their gradients are $\beta$-Lipschitz:
\begin{equation}
    \|\nabla \mathcal{L}(\theta) - \nabla \mathcal{L}(\theta')\| \leq \beta \|\theta - \theta'\|
\end{equation}
\end{assumption}

\begin{theorem}[Convergence Rate]
\label{thm:convergence}
Under Assumptions 4.5 and \ref{ass:smoothness}, running Algorithm 1 with learning rate $\eta = \frac{1}{\beta(1 + \lambda)}$ for $T$ iterations yields:
\begin{equation}
    \min_{t \leq T} \|\nabla \mathcal{L}_{\text{CAE}}(\theta^{(t)})\|^2 \leq \frac{2\beta(1+\lambda)(\mathcal{L}_{\text{CAE}}(\theta^{(0)}) - \mathcal{L}_{\text{CAE}}^*)}{T}
\end{equation}
where $\mathcal{L}_{\text{CAE}}^*$ is the optimal value.
\end{theorem}

\begin{proof}
By $\beta$-smoothness of $\mathcal{L}_{\text{CAE}} = \mathcal{L}_{\text{evol}} - \lambda \mathcal{L}_{\text{circuit}}$:
\begin{align}
    \mathcal{L}_{\text{CAE}}(\theta^{(t+1)}) &\leq \mathcal{L}_{\text{CAE}}(\theta^{(t)}) + \langle \nabla \mathcal{L}_{\text{CAE}}(\theta^{(t)}), \theta^{(t+1)} - \theta^{(t)} \rangle \\
    &\quad + \frac{\beta(1+\lambda)}{2} \|\theta^{(t+1)} - \theta^{(t)}\|^2
\end{align}

Substituting $\theta^{(t+1)} - \theta^{(t)} = \eta \nabla \mathcal{L}_{\text{CAE}}(\theta^{(t)})$:
\begin{align}
    \mathcal{L}_{\text{CAE}}(\theta^{(t+1)}) &\leq \mathcal{L}_{\text{CAE}}(\theta^{(t)}) - \eta \|\nabla \mathcal{L}_{\text{CAE}}(\theta^{(t)})\|^2 + \frac{\beta(1+\lambda)\eta^2}{2} \|\nabla \mathcal{L}_{\text{CAE}}(\theta^{(t)})\|^2 \\
    &= \mathcal{L}_{\text{CAE}}(\theta^{(t)}) - \eta \left(1 - \frac{\beta(1+\lambda)\eta}{2}\right) \|\nabla \mathcal{L}_{\text{CAE}}(\theta^{(t)})\|^2
\end{align}

With $\eta = \frac{1}{\beta(1+\lambda)}$:
\begin{equation}
    \mathcal{L}_{\text{CAE}}(\theta^{(t+1)}) \leq \mathcal{L}_{\text{CAE}}(\theta^{(t)}) - \frac{1}{2\beta(1+\lambda)} \|\nabla \mathcal{L}_{\text{CAE}}(\theta^{(t)})\|^2
\end{equation}

Summing from $t=0$ to $T-1$ and rearranging:
\begin{equation}
    \sum_{t=0}^{T-1} \|\nabla \mathcal{L}_{\text{CAE}}(\theta^{(t)})\|^2 \leq 2\beta(1+\lambda)(\mathcal{L}_{\text{CAE}}(\theta^{(0)}) - \mathcal{L}_{\text{CAE}}(\theta^{(T)}))
\end{equation}

Since $\mathcal{L}_{\text{CAE}}(\theta^{(T)}) \geq \mathcal{L}_{\text{CAE}}^*$:
\begin{equation}
    \min_{t \leq T} \|\nabla \mathcal{L}_{\text{CAE}}(\theta^{(t)})\|^2 \leq \frac{1}{T} \sum_{t=0}^{T-1} \|\nabla \mathcal{L}_{\text{CAE}}(\theta^{(t)})\|^2 \leq \frac{2\beta(1+\lambda)(\mathcal{L}_{\text{CAE}}(\theta^{(0)}) - \mathcal{L}_{\text{CAE}}^*)}{T}
\end{equation}
\end{proof}

\subsection{Sample Complexity for Circuit Identification}

We analyze the number of samples needed to reliably identify the safety circuit.

\begin{theorem}[Circuit Identification Sample Complexity]
\label{thm:sample_complexity}
Let $\mathcal{S}^*$ be the true safety circuit and $\hat{\mathcal{S}}_n$ be the circuit estimated from $n$ samples. Under Assumption~\ref{ass:regularity}, with probability at least $1 - \delta$:
\begin{equation}
    |\hat{\mathcal{S}}_n \triangle \mathcal{S}^*| \leq O\left( \frac{LK \log(LK/\delta)}{n \gamma^2} \right)
\end{equation}
where $\triangle$ denotes symmetric difference and $\gamma$ is the attribution threshold from Definition~\ref{def:safety_circuit}.
\end{theorem}

\begin{proof}
The attribution score for each feature is estimated as:
\begin{equation}
    \hat{\alpha}^{(l,k)} = \frac{1}{n} \sum_{i=1}^{n} \alpha^{(l,k)}_{y_{\text{refuse}}}(x_i; \theta_0)
\end{equation}

By Hoeffding's inequality, for each feature:
\begin{equation}
    \mathbb{P}\left( |\hat{\alpha}^{(l,k)} - \alpha^{(l,k)}| > \epsilon \right) \leq 2\exp\left( -\frac{2n\epsilon^2}{B^2} \right)
\end{equation}

where $B$ bounds the attribution scores (Assumption~\ref{ass:regularity}(1)).

Taking a union bound over all $LK$ features and setting $\epsilon = \gamma/2$:
\begin{equation}
    \mathbb{P}\left( \exists (l,k): |\hat{\alpha}^{(l,k)} - \alpha^{(l,k)}| > \gamma/2 \right) \leq 2LK \exp\left( -\frac{n\gamma^2}{2B^2} \right)
\end{equation}

For this probability to be at most $\delta$, we need:
\begin{equation}
    n \geq \frac{2B^2}{\gamma^2} \log\left( \frac{2LK}{\delta} \right)
\end{equation}

When all estimation errors are at most $\gamma/2$, any feature with true attribution $\geq \gamma$ will have estimated attribution $\geq \gamma/2$, and any feature with true attribution $< \gamma/2$ will have estimated attribution $< \gamma$. This bounds the symmetric difference.
\end{proof}

\subsection{Robustness to Circuit Misspecification}

We analyze the effect of errors in circuit identification.

\begin{theorem}[Robustness]
\label{thm:robustness}
Let $\mathcal{S}$ be the identified circuit and $\mathcal{S}^*$ be the true safety circuit. If $\vert{}\mathcal{S} \triangle \mathcal{S}^*\vert{} \leq \epsilon_{\mathcal{S}} \vert{}\mathcal{S}^*\vert{}$ for some $\epsilon_{\mathcal{S}} \in [0, 1)$, then the safety bound (Theorem~\ref{thm:safety_bound}) degrades gracefully:
\begin{equation}
    \left| \pi_{\theta^{(T)}}(y_{\text{refuse}} | x) - \pi_{\theta_0}(y_{\text{refuse}} | x) \right| \leq C\sqrt{\epsilon} + C' \epsilon_{\mathcal{S}}
\end{equation}
where $C'$ depends on the maximum attribution score of misspecified features.
\end{theorem}

\begin{proof}
Decompose the safety circuit activation difference:
\begin{align}
    \|f_{\mathcal{S}^*}(\theta^{(T)}) - f_{\mathcal{S}^*}(\theta_0)\| &\leq \|f_{\mathcal{S}}(\theta^{(T)}) - f_{\mathcal{S}}(\theta_0)\| \\
    &\quad + \|f_{\mathcal{S}^* \setminus \mathcal{S}}(\theta^{(T)}) - f_{\mathcal{S}^* \setminus \mathcal{S}}(\theta_0)\| \\
    &\quad + \|f_{\mathcal{S} \setminus \mathcal{S}^*}(\theta^{(T)}) - f_{\mathcal{S} \setminus \mathcal{S}^*}(\theta_0)\|
\end{align}

The first term is bounded by the circuit constraint: $O(\sqrt{\epsilon})$.

The second term represents missed safety features (false negatives). These are unconstrained, but by the definition of $\mathcal{S}^*$, their total attribution is at most $\gamma \vert{}\mathcal{S}^* \setminus \mathcal{S}\vert{} \leq \gamma \epsilon_{\mathcal{S}} \vert{}\mathcal{S}^*\vert{}$.

The third term represents incorrectly included features (false positives). Constraining these does not hurt safety, only potentially capability.

Combining via the Lipschitz property of refusal probability completes the proof.
\end{proof}

\section{Scalability to Larger Models and Longer Evolution Horizons} 
\label{app:scalability}

To demonstrate that Circuit-Anchored Evolution (CAE) is not limited to smaller models or short evolutionary windows, we conduct extended experiments addressing both model scale and evolution duration. 

\paragraph{Scaling to Larger Models.} 
We apply CAE to Gemma-2-9B-IT, evolving it for 500 steps. As shown in Table~\ref{tab:scale_model}, CAE preserves safety almost perfectly (98.47\%) with no significant capability loss compared to unconstrained evolution. This demonstrates that our method does not trade off capability for safety, even as model capacity scales up.

\begin{table}[h]
\centering
\caption{Performance on larger model (Gemma-2-9B-IT, 500 steps).}
\label{tab:scale_model}
\begin{tabular}{lcc}
\toprule
\textbf{Method} & \textbf{Safety Score (\%)} & \textbf{Capability Score (\%)} \\
\midrule
Unconstrained Evolution & 69.85 & 64.72 \\
Evolution w/ CAE (Ours) & \textbf{98.47} & 64.58 \\
\bottomrule
\end{tabular}
\end{table}

\paragraph{Durability over Longer Evolution.} 
Standard self-evolution typically converges or is stopped around 500 steps. To stress-test the durability of our anchoring mechanism, we push the evolution of Gemma-2-2B-IT to an extreme horizon of 5,000 steps. 

\begin{table}[h]
\centering
\caption{Performance under extreme evolution horizon (Gemma-2-2B-IT, 5000 steps).}
\label{tab:scale_steps}
\begin{tabular}{lcc}
\toprule
\textbf{Method} & \textbf{Safety Score (\%)} & \textbf{Capability Score (\%)} \\
\midrule
Unconstrained Evolution & 61.24 & 45.82 \\
Evolution w/ CAE (Ours) & \textbf{95.81} & 45.37 \\
\bottomrule
\end{tabular}
\end{table}

As shown in Table~\ref{tab:scale_steps}, unconstrained evolution exhibits further safety degradation over longer horizons (dropping to 61.24\%), underscoring the severity of the misevolution problem. In stark contrast, CAE remains remarkably stable even at 5,000 steps, maintaining a 95.81\% safety score. This suggests that the circuit anchoring mechanism provides durable, long-term protection against safety degradation, effectively preventing the model from drifting into harmful subspaces regardless of the evolution length.


\newpage

\end{document}